\def\year{2021}\relax

\documentclass[letterpaper]{article} 
\usepackage{aaai21}  
\usepackage{times}  
\usepackage{helvet} 
\usepackage{courier}  
\usepackage[hyphens]{url}  
\usepackage{graphicx} 
\usepackage{natbib}  
\usepackage{caption} 
\usepackage{balance} 

\usepackage[utf8]{inputenc} 
\usepackage[T1]{fontenc}    
\usepackage{hyperref}       
\usepackage{url}            
\usepackage{booktabs}       
\usepackage{amsfonts}       
\usepackage{nicefrac}       
\usepackage{microtype}      
\usepackage{xcolor}         

\usepackage{graphicx}
\usepackage{epstopdf}
\usepackage{mathtools}
\usepackage{amsthm, amssymb}
\usepackage[ruled, vlined, linesnumbered]{algorithm2e}
\usepackage{caption, subcaption}
\usepackage{enumitem}
\usepackage{color}
\usepackage{comment}

\newtheorem{theorem}{Theorem}
\newtheorem{lemma}[theorem]{Lemma}

\newtheorem{corollary}[theorem]{Corollary}

\newcommand{\abs}[1]{\left| #1 \right|}

\newcommand{\D}{\mathcal{D}}
\newcommand{\E}{\mathcal{E}}
\newcommand{\bE}{\mathbf{E}}
\newcommand{\N}{\mathcal{N}}

\newcommand{\TSD}{{\tt{BTSD}}}
\newcommand{\TSI}{{\tt{BTSI}}}
\newcommand{\TSDv}{{\tt{BTSD-}}}
\newcommand{\TSIv}{{\tt{BTSI-}}}
\newcommand{\TS}{{\tt{TS}}}
\newcommand{\UCB}{{\tt{UCB}}}
\newcommand{\BSE}{{\tt{BSE}}}
\newcommand{\BAE}{{\tt{BAE}}}
\newcommand{\UCBI}{{\tt{UCBI}}}
\newcommand{\KMS}{{\tt{KMS}}}

\title{Batched Thompson Sampling for Multi-Armed Bandits\thanks{N. Karpov and Q. Zhang are supported in part by CCF-1844234 and CCF-2006591.}}

\author {
	Nikolai Karpov \quad \quad
	Qin Zhang  \\
}
\affiliations {
	 Indiana University Bloomington \\
	\{nkarpov, qzhangcs\}@indiana.edu
}

\begin{document}

\maketitle

\begin{abstract}
We study Thompson Sampling algorithms for stochastic multi-armed bandits in the batched setting, in which we want to minimize the regret over a sequence of arm pulls using a small number of policy changes (or, batches).  We propose two algorithms and demonstrate their effectiveness by experiments on both synthetic and real datasets.  We also analyze the proposed algorithms from the theoretical aspect and obtain almost tight regret-batches tradeoffs for the two-arm case.
\end{abstract}

\section{Introduction}
\label{sec:intro}

Thompson sampling (TS), first proposed by Thompson in 1933~\cite{Thompson33}, is a classic algorithm for solving online decision problems. It was not until a decade ago that Thompson sampling has regained people's attention; some empirical studies \cite{Scott10,CL11} showed that in several applications, including online advertisement and article recommendation, TS-based algorithms outperform alternatives such as {\em upper confidence bound} (UCB) based algorithms. Since then Thompson sampling has been used extensively in algorithms for Internet advertising~\cite{GCBH10}, online recommendation~\cite{KBK+15}, Monte Carlo tree search~\cite{BWC13}, revenue management~\cite{FSW18}, website optimization~\cite{HNL+17}, etc.  The method has also been studied from the theoretical aspect, most notably for stochastic multi-armed bandit (MAB), for which tight bounds on the regret have been obtained~\cite{AG12,KKM12,AG17,JXS+20}.
 
In practical settings, it is desirable to make the algorithm have a small number of policy changes. This is because in many applications the learning performance would be much better if the data observations are processed in {\em batches}, each of which under a fixed policy.  For example,  in clinical trials, patients are usually tested in groups and the results are released in batches; and in crowd-sourcing, questions and answers are typically sent and received in batches.  Another obvious advantage of batched algorithms is that such algorithms are naturally parallelizable.

In this paper, we give the first study of Thompson sampling in the batched setting using stochastic multi-armed bandit as the vehicle.  We aim to establish a theoretical foundation for batched Thompson sampling and demonstrate its superior practical performance.

\vspace{2mm}
\noindent{\bf Multi-Armed Bandits.\ \ }  In MAB we have $N$ arms, denoted by $\{1, 2, \ldots, N\}$. Each arm $i$ is associated with a reward distribution with support in $[0, 1]$ and (unknown) mean $\mu_i$.   Let $\mu_* = \max_{i \in [N]} \mu_i$.\footnote{We use $[Z]$ to denote the set $\{1, 2, \ldots, Z\}$.}.   We want to make a sequence of $T$ pulls over the $N$ arms with the goal of maximizing the accumulated reward. Each time we pull an arm, we get a reward sampled from its associated distribution. 

More precisely, for each time $t \in [T]$, let $\pi_t \in [N]$ be the arm pulled at time $t$, and let $\pi = \{\pi_1, \pi_2, \ldots, \pi_T\}$ denote the sequence of $T$ pulls.  We define the regret with respect to the pull sequence $\pi$ to be 
\begin{equation*}
\text{Reg}_T(\pi) \triangleq \sum_{t=1}^T(\mu_\star - \mu_{\pi_t}) = \sum_{t=1}^T \Delta_{\pi_t},
\end{equation*} 
where $\Delta_i = \mu_* - \mu_i$ is the gap between the mean of the best arm and that of the $i$-th arm.  $\text{Reg}_T(\pi)$ can be seen as the difference between the reward of the optimal algorithm which always pulls the best arm and that of the algorithm taking the pull sequence $\pi$.

In the batched setting, the sequence of $T$ pulls is partitioned into a set of batches.  Let $0 = T^{(0)} < T^{(1)}  < \ldots < T^{(m)} = T$.  The $r$-th batch contains $(T^{(r+1)} - T^{(r)})$ pulls, which are fully determined at the beginning of the $r$-th batch. That is, they only depend on the indices of arms being pulled and outcomes of the pulls before time $T^{(r)}$.  Our goal is to design an algorithm that uses a small number of batches {\em and} achieves a small expected regret $\bE[\text{Reg}_T(\pi)]$, where the expectation is taken over the randomness of the algorithm and the arm pulls. 

\vspace{2mm}
\noindent{\bf Thompson Sampling.\ }  In Thompson sampling (TS), at any time step $t$, there is a prior distribution $\D_i(t)$ on each arm $i$. At each time step, we sample for each $i \in [N]$ a value $\theta_i(t) \sim \D_i(t)$, and pull the arm associated with the largest sampled value, that is, $\pi_t \gets \arg\max_{i \in [N]} \theta_i(t)$.  

In the setting where the sequence of pulls are fully adaptive, after each time step $t$, we use the obtained pull outcome (reward) to compute the posterior distribution of $\D_{\pi_t}(t)$, which is the new prior distribution of arm $i$ for the next pull. The prior distributions of other arms $j \neq \pi_t $ remains unchanged.  In the batched setting, we can only update the prior of arm $i$ for each $i \in [N]$ at the end of each batch using the information obtained in the current batch.

\vspace{2mm}
\noindent{\bf Our Contributions.} In this paper, we propose two batched algorithms for regret minimization in MAB using Thompson sampling.  Our first algorithm is designed targeting an {\em instance-dependent} bound on the regret, that is, the regret is a function of the number of arms $N$, the gaps $\Delta_i$, and the time horizon $T$.  Our second algorithm targets an {\em instance-independent} bound on the regret, that is, the regret only depends on the time horizon $T$ and the number of arms $N$. 

As previously stated, Thompson sampling often has a significant advantage over UCB-based algorithms in terms of practical performance. We have conducted an extensive set of experiments on both synthetics and real datasets (Section~\ref{sec:exp}). Our experiments show that the proposed algorithms significantly outperform existing (non-TS-based) algorithms in the batched setting.  Our instance-dependent algorithm even outperforms the fully adaptive Thompson sampling algorithm on many datasets.  We consider the empirical performance of our algorithms the main contribution of this paper.

Our theoretical bounds are conducted for the special case when there are two arms (i.e., $N = 2$).  In the instance-dependent case (Section~\ref{sec:dependent}), our algorithm uses $\log T$ batches and achieves a regret of $O({\log T}/{\Delta})$, where $\Delta$ is the gap between the mean of the best arm and that of the other arm.  In the instance-independent case (Section~\ref{sec:independent}), our algorithm uses $\log\log T$ batches and achieves a regret of $O(\sqrt{T\ln T})$.  Based on the lower bound results by Gao et al.~\cite{GHRZ19}, these tradeoffs are almost tight.\footnote{In the case of two arms, Gao et al.\ showed that $\Omega(\log T/\log\log T)$ batches is needed to achieve instance-dependent regret $O(\log T/\Delta)$, and $\Omega(\log\log T)$ batches is needed to achieve instance-independent regret $O(\sqrt{T})$.}

\vspace{2mm}
\noindent{\bf Related Work.}  We now briefly survey previous work related to this paper. Thompson sampling algorithms have been studied for regret minimization in MAB in the sequential setting.
Agrawal and Goyal~\cite{AG12} gave a TS-based algorithm that achieves an instance-dependent upper bound on the expected regret, which was then improved by Kaufmann et al.~\cite{KKM12} to be asymptotically optimal.  Agrawal and Goyal~\cite{AG17} proposed a TS-based algorithm with almost tight regret bounds in the instance-independent setting under Beta and Gaussian priors. Very recently, Jin et al.~\cite{JXS+20} obtained an optimal instance-independent regret upper bound $O(\sqrt{NT})$. All these algorithms are fully adaptive.

In recent years, batched algorithms for regret minimization in MAB have received considerable attention.  Perchet et al.~\cite{PRCS15} gave a set of batched algorithms for inputs with two arms.  Their algorithm for the instance-independent case achieves $O(\log\log T)$ batches and that for the instance-dependent case achieves $O(\log T)$ batches; both algorithms achieve almost optimal regret (up to logarithmic factors).  These results were later generalized to the $N$-arm case by Gao et al.~\cite{GHRZ19} and Esfandiari et al.~\cite{EKMM19}.  Though not explicitly stated, the earlier algorithm UCB2~\cite{ACF02} can be implemented in the batched model using $O(N \log T)$ batches, and the improved-UCB~\cite{AO10} can be implemented in the batched model using $O(\log T)$ batches.
 
Finally, we note that batched algorithms have been studied for other basic problems in reinforcement and online learning, including best and top-$k$ arm identifications in MAB~\cite{JJNZ16,AAAK17,JSXC19}, multinomial logit bandit~\cite{DLZZ20}, online learning and convex optimization~\cite{CDS13,DRY18}, and Q-learning~\cite{BXJW19}.

Very recently, concurrent of our work, Karbasi et al.~\cite{KMS21} studied the same problem (i.e., batched Thompson sampling for regret minimization in MAB).\footnote{The first draft of this paper appeared in May 2021, before  \cite{KMS21} was posted on arXiv.} They proposed an instance-independent algorithm with $O(N \log T)$ batches and regret $O(\sqrt{NT \ln T})$, and an instance-dependent algorithm with almost optimal regret and $O(N \log T)$ batches. In the high level, their algorithms can be seens as adding the doubling technique to standard TS algorithms for each arm, which can make sure that the ``lagged'' prior distribution does not deviate too much from the actual prior distribution at each step of the batched algorithm. Compared with algorithms in \cite{KMS21}, ours have the following advantages: (1) For $N = 2$, our tradeoffs are almost tight, while the batch complexity in the instance-independent algorithm in \cite{KMS21} is exponentially larger than the optimal bound. (2) We presented full tradeoffs between number of batches and regrets, while the analysis in \cite{KMS21} is only conducted for end points. (3) Most importantly, the practical performance of our algorithms are significantly better than those in \cite{KMS21} in terms of the number of batches for general $N$. We have included in Section~\ref{sec:exp} of this version a detailed empirical comparison with the batched Thompson Sampling algorithm in \cite{KMS21} under Gaussian priors.


\section{Instance-Dependent Algorithm}
\label{sec:dependent}

We start by introducing some useful notations.
For each batch $r$, let $T^{(r)}$ be the total number of pulls on the $N$ arms {\em before} the $r$-th batch, and $T_i^{(r)}$ be the number of pulls on arm $i$ before the $r$-th batch.  For each $i \in [N]$ and batch $r$, let $\mu_i^{(r)}$ be the {\em empirical mean} of the $i$-th arm after $T_i^{(r)}$ pulls.

Our batched TS-based algorithm with instance-dependent bounds is described in Algorithm~\ref{alg:dependent}.  Let us describe the algorithm briefly in words.   At the beginning of the $r$-th batch, we compute for each arm $i$ the probability that its corresponding random variable $X_i \sim \N\left(\mu^{(r)}_i, \alpha / T^{(r)}_i\right)$ is larger than that of other surviving arms $j \in I^{(r)}$; we denote this value by $q_i^{(r)}$. Let $q^{(r)}$ be the largest value in $\{q_i^{(r)}\}_{i \in I^{(r)}}$. We construct set $I^{(r+1)}$ by discarding arms in $I^{(r)}$ whose associated probability $q_i^{(r)}$ is smaller than $q^{(r)}/\beta$.  The length of the $r$-th batch is set to be $\gamma^r \abs{I^{(r+1)}}$ where $\gamma = T^{1/M}$ is a parameter to control the number of batches, except for the last batch whose length is set to be $(T - T^{(r)})$.  Once the batch size is determined, the number of pulls on each surviving arm $i$ in the $r$-th batch is proportional to its associated probability $q_i^{(r)}$.


\begin{algorithm}[t]
\caption{Batched-TS-D$([N], T, \alpha, \beta, M)$}
\label{alg:dependent}
\KwIn{A set of $N$ arms $\{1, \ldots, N\}$, time horizon $T$, \#batches $M$, parameters $\alpha, \beta$}
	Initialize $I^{(1)}\gets I$, $\gamma \gets T^{1/M}$ \; 
	pull each arm once; $\forall i \in [N]$, $T^{(1)}_i \gets 1$, and $\mu^{(1)}_i \gets \text{reward of pulling arm $i$}$;  $T^{(1)} \gets N$ \;
	\For{$r = 1, 2, \dotsc, M$}
	{
		\lIf{$T^{(r)} = T$}{terminate the process}
		for each $i \in I^{(r)}$, let $X^{(r)}_i$ denote the random variable distributed according to  $\N\left(\mu^{(r)}_i, \alpha / T^{(r)}_i\right)$ \;
		\lForEach{$i \in I^{(r)}$}
		{
			$q^{(r)}_i \gets \Pr\left[X^{(r)}_i > \max\limits_{j \in I^{(r)}, j \neq i} \left\{X^{(r)}_j\right\}\right]$ 
		}
		$q^{(r)} \gets \max\limits_{i \in I^{(r)}} \left\{q_i^{(r)}\right\}$ \;
		$I^{(r + 1)} \gets \left\{ q^{(r)}_i \ge q^{(r)} / \beta \mid i \in I^{(r)}\right\}$ \label{ln:prune} \;
		$T^{(r+1)} \gets \min\left\{T^{(r)}+ \abs{I^{(r+1)}} \cdot \gamma^r, T\right\}$ \;
		\lForEach{$i \in I^{(r + 1)}$}{$\hat{q}^{(r)}_i = q^{(r)}_i / \sum\limits_{j \in I^{(r + 1)}} q^{(r)}_j$}
		\ForEach{$i \in I^{(r)}$}
		{
			pull arm $i$ for $\hat{q}^{(r)}_i \cdot \left(T^{(r+1)} - T^{(r)}\right)$ times; set $T_i^{(r+1)} \gets T_i^{(r)} + \hat{q}^{(r)}_i \cdot \left(T^{(r+1)} - T^{(r)}\right)$ \;
			compute $\mu_i^{(r+1)}$ using $\mu_i^{(r)}$, $T_i^{(r)}$, and the pull outcomes on arm $i$ in the current batch.	
		}
	}
\end{algorithm}

We would like to make a few points, which also apply to the instance-independent algorithm in Section~\ref{sec:independent}.
\begin{enumerate}
\item We do not perform the standard Thompson sampling (TS) at each time step, but compute for each arm the probability that it gives the maximum reward  based on the prior distributions of the surviving arms {\em once} at the beginning of each batch.\footnote{In practice this can be computed efficiently by numerical methods.}  In other words, we set the number of pulls for arm $i$ in each batch to be the {\em expected} number of pulls if we follow the standard TS algorithm with the mean and variance being the empirical ones at the beginning of the round.  If a batch is large enough, then using the expectation has similar performance than sampling at each step following the standard TS algorithm.

\item  The pruning step (Line~\ref{ln:prune}) is added for the convenience of the analysis.  In Section~\ref{sec:exp}, we will show that removing this line, which brings the algorithm closer to the standard Thompson sampling algorithm, has little affect on the performance of the algorithm.

\item  Algorithm~\ref{alg:dependent} works for general $N$ (arms), though in the analysis below we only prove theoretical bounds in the case when $N = 2$.  In Section~\ref{sec:exp}, experiments are also conducted for $N > 2$.
\end{enumerate}

\noindent{\bf Analysis.\ \ }
For $N = 2$,  w.l.o.g., we assume arm $1$ is the best arm (but, of course, the algorithm does not know this at the beginning). Let $\Delta \triangleq \Delta_2 = \mu_1 - \mu_2$.


\begin{theorem}
\label{thm:dependent}
Setting $\alpha = \ln(2T)$ and $\beta = 100$, Batched-TS-D$(2, T, \alpha, \beta, M)$ (Algorithm~\ref{alg:dependent}) has expected regret $O\left(T^{\frac{1}{M}} \ln T / \Delta\right)$ and uses at most $M$ batches.
\end{theorem}

Setting $M = \log T$, we have the following corollary.
\begin{corollary}
\label{cor:dependent}
Batched-TS-D$(2, T, \alpha, \beta, \log T)$ (Algorithm~\ref{alg:dependent}) has expected regret $O\left({\ln T}/{\Delta}\right)$ and uses at most $\log T$ batches.
\end{corollary}

In this section we prove Theorem~\ref{thm:dependent}. We make use of the following standard concentration inequality.

\begin{lemma}[Hoeffding’s inequality]
	\label{lem:chernoff}
	Let $X_1, \dotsc, X_n \in [0, 1]$ be independent random variables and $X = \sum\limits_{i = 1}^n X_i$.
	Then
	\[
	\Pr[X > \bE[X] + t] \leq \exp\left(-{2 t^2}/{n}\right)
	\quad
	\]
	and
	\[
	\quad
	\Pr[X < \bE[X] - t] \leq \exp\left(-{2 t^2}/{n}\right)
	\,.
	\]
\end{lemma}Define the following event which we will condition on in the rest of the proof.
\begin{equation*}
\label{eq:E}
		\E \triangleq  \left\{\forall{i} \in \{1, 2\}, \forall{t \in [T]} : \abs{\mu_i - {\mu}_i(t)} \le \sqrt{\frac{\ln(2T)}{t}}\right\}\,.
\end{equation*}

\begin{lemma}
\label{lem:event-E}
$\Pr\left[\E\right] \ge 1 - {1}/{T}$.
\end{lemma}
\begin{proof}
	By Hoeffding's inequality (Lemma~\ref{lem:chernoff}), for any $i \in [N]$ and $t \in [T]$, we have 
	\begin{equation}
		\label{eq:probe}
		\Pr\left[\abs{\mu_i - {\mu}_i(t)} \ge \sqrt{\frac{\ln(NT)}{t}}\right] \le 2 e^{-2 \ln(NT)} \le \frac{2}{(NT)^2}.
	\end{equation}
	The lemma follows by applying a union bound on (\ref{eq:probe}) over all $i \in [N]$ and $t \in [T]$.
\end{proof}

We write 
\begin{eqnarray}
&& \bE[\text{Reg}_T(\pi)] \nonumber \\
&=& \bE[\text{Reg}_T(\pi)\ |\ \E] \Pr[\E] + \bE[\text{Reg}_T(\pi)\ |\ \bar{\E}] \Pr[\bar{\E}] \nonumber \\
&\le& \bE[\text{Reg}_T(\pi)\ |\ \E] + T \cdot (1 - \Pr[\E]) \nonumber \\
&=& \bE[\text{Reg}_T(\pi)\ |\ \E] + 1. \label{eq:b-1}
\end{eqnarray}
We thus only need to bound $\bE[\text{Reg}_T(\pi)\ |\ \E]$.   The next lemma indicates that arm $1$ (i.e., the best arm) will never be pruned at Line~\ref{ln:prune} of Algorithm~\ref{alg:dependent} during the process.

\begin{lemma}
\label{lem:X1}
For any $r$ we have $\Pr[X^{(r)}_1 > X^{(r)}_2] > 1/\beta$.
\end{lemma} 

\begin{proof}
Conditioned on $\E$, we have for any $i \in \{1, 2\}$, it holds that 
$\abs{\mu^{(r)}_i  - \mu_i} \le \sqrt{\frac{\ln(2T)}{T^{(r)}_i}}$, and		
\begin{eqnarray}
		\mu^{(r)}_1 - \mu^{(r)}_{2} &\ge& \mu_{1} - \mu_{2} - \sqrt{\frac{\ln(2T)}{T^{(r)}_1}} - \sqrt{\frac{\ln(2T)}{T^{(r)}_2}} \nonumber \\
		 &\ge& - \left(\sqrt{\frac{\ln(2T)}{T^{(r)}_1}} + \sqrt{\frac{\ln(2T)}{T^{(r)}_2}}\right),  \label{eq:gap}
\end{eqnarray}
where the second inequality is due to the fact (our assumption) that arm $1$ is the best arm. 
		Using the fact that $\forall x, y >0, x + y \le \sqrt{2} \cdot \sqrt{x^2 + y ^2}$, we have
		\begin{equation}
		\label{eq:a-1}
		\mu^{(r)}_1 - \mu^{(r)}_{2} \ge  - \sqrt{2} \cdot \sqrt{\frac{\ln(2T)}{T^{(r)}_1} + \frac{\ln(2T)}{T^{(r)}_2}}.
		\end{equation}
		Let $Z \triangleq X^{(r)}_1 - X^{(r)}_2$. Recall that $X^{(r)}_1 \sim \N\left(\mu^{(r)}_1, \alpha / T^{(r)}_1\right)$ and  $X^{(r)}_2 \sim \N\left(\mu^{(r)}_2, \alpha / T^{(r)}_2\right)$. We have 
		$$\Pr[X^{(r)}_1 > X^{(r)}_2] = \Pr[Z > 0],$$ where 
		\begin{equation}
		\label{eq:a-2}
		Z \sim \N\left(\mu^{(r)}_1 - \mu^{(r)}_2, \frac{\alpha}{T^{(r)}_{1}} + \frac{\alpha}{T^{(r)}_2}\right).
		\end{equation}
		Let $\Phi(x) = \frac{1}{\sqrt{2\pi}} \int_{-\infty}^{x}  e^{-\frac{x^2}{2}}dx$ be the cumulative distribution function (CDF) of $\N(0, 1)$.  By (\ref{eq:a-1}) and (\ref{eq:a-2}), we have
		\begin{eqnarray*}
			\Pr[Z > 0] &=& \Phi\left(\frac{\mu^{(r)}_1 - \mu^{(r)}_{2}}{\sqrt{\frac{\alpha}{T^{(r)}_1} + \frac{\alpha}{T^{(r)}_2}}}\right) \\
			&\ge& \Phi\left(-\sqrt{2} \cdot \sqrt{\frac{\alpha}{\ln(2T)}}\right) \\ &\geq& \Phi\left(-\sqrt{2}\right) > 0.01 = 1/\beta,
		\end{eqnarray*}
		where the first inequality is due to the monotonicity of the $\Phi(\cdot)$ function.
\end{proof}

The next lemma states that after being pulled for a sufficient number of times, arm $2$ will be pruned at Line~\ref{ln:prune}.

\begin{lemma}
\label{lem:pull-bound}
		When $T^{(r)}_2 > \frac{100 \beta \ln(2T)}{\Delta^2}$ for some $r \ge 1$, we have $\Pr[X^{(r)}_2 > X^{(r)}_1] < {1}/{\beta}$.
\end{lemma}

\begin{proof}
First, by the pruning condition at Line~\ref{ln:prune} of Algorithm~\ref{alg:dependent}, in any batch $r$ when there are still two arms left, we have  
\begin{equation}
\label{eq:c-1}
\frac{q_1^{(r)}}{q_2^{(r)}} \in \left[\frac{1}{\beta}, \beta\right] \quad \text{or} \quad \frac{T_1^{(r)}}{T_2^{(r)}} \in \left[\frac{1}{\beta}, \beta\right]. 
\end{equation}

If $T^{(r)}_2 > \frac{100 \beta \ln(2T)}{\Delta^2}$, then by (\ref{eq:c-1}) we have $T^{(r)}_1 > \frac{100 \ln(2T)}{\Delta^2}$.  Conditioned on event $\E$, we have
$\abs{\mu^{(r)}_i - \mu_i} \le \frac{\Delta}{10}$ for $i = 1, 2$.
Since $\mu_1 - \mu_2 = \Delta$, we have 
\begin{equation}
\label{eq:c-3}
\mu^{(r)}_1 - \mu^{(r)}_2 \ge \frac{8}{10} \Delta.
\end{equation}

Let $Z \triangleq X^{(r)}_2 - X^{(r)}_1$. We have $\Pr[X^{(r)}_2 > X^{(r)}_1] = \Pr[Z > 0]$, where 
\begin{equation}
\label{eq:c-4}
		Z \sim \N\left(\mu^{(r)}_2 - \mu^{(r)}_1, \frac{\alpha}{T^{(r)}_{1}} + \frac{\alpha}{T^{(r)}_2}\right).
\end{equation}
By (\ref{eq:c-3}) and (\ref{eq:c-4}), we get
\begin{eqnarray*}
			\Pr[Z > 0] &=& \Phi\left(\frac{\mu^{(r)}_2 - \mu^{(r)}_1}{\sqrt{\alpha / T^{(r)}_1 + \alpha / T^{(r)}_2}}\right) \\
			&\le& \Phi\left(-\frac{8 \Delta }{10} \cdot \frac{1}{\sqrt{2} (\Delta / 10)}\right) \\
			&=& \Phi\left(-\frac{8}{\sqrt{2}}\right) < 10^{-4} < 1/\beta,
\end{eqnarray*}
where in the first inequality we have used the fact that $T^{(r)}_1, T^{(r)}_2 > \frac{100\ln(2T)}{\Delta^2}$.
\end{proof}

Now we are ready to bound the expected regret.  It is easy to see from the description of the algorithm that for any $r$ we have $\gamma^{(r+1)}/(\beta+1) \le T^{(r)}_i \le \gamma^{(r+1)}$.  We thus have 
\begin{equation}
\label{eq:d-1}
T^{(r+1)}_i / T^{(r)}_i \le 2\gamma \beta.
\end{equation}
Consider the largest index $r$ such that $T^{(r)}_2 < \frac{100\beta \ln(2T)}{\Delta^2}$. We have  $T^{(r+1)}_2 \ge \frac{100\beta \ln(2T)}{\Delta^2}$, and thus arm $2$ will be pruned in the $(r+1)$-th batch.  
By (\ref{eq:d-1}) we also have $T^{(r+1)}_2 \le 2\gamma \beta T^{(r)}_2 \le \frac{200\gamma \beta^2 \ln(2T)}{\Delta^2}$.  Therefore, we can bound
\begin{equation}
\label{eq:d-2}
\bE[\text{Reg}_T(\pi)\ |\ \E] \le \Delta T^{(r)}_2 \le \frac{200\gamma \beta^2 \ln(2T)}{\Delta}.  
\end{equation}
Plugging (\ref{eq:d-2}) to (\ref{eq:b-1}) we get
$\bE[\text{Reg}_T(\pi)] = O\left(\frac{T^{1/M} \ln T}{\Delta}\right)$  (recall that we have set $\beta = 100$ and $\gamma = T^{1/M}$).  


\section{Instance-Independent Algorithm}
\label{sec:independent}

In this section, we present a batched TS-based algorithm whose regret is independent of the input instance.
We will reuse the notations $T^{(r)}, T^{(r)}_i$, and $\mu^{(r)}_i$ defined in Section~\ref{sec:dependent}.  The algorithm is described in Algorithm~\ref{alg:independent}.

Let us describe Algorithm~\ref{alg:independent} briefly in words.  Same as Algorithm~\ref{alg:dependent}, at the beginning of each batch $r$, we compute for each $i \in [N]$ the probability $q^{(r)}_i$ such that the sample from distribution $\N\left(\mu^{(r)}_i, \alpha / T^{(r)}_i\right)$ is larger than samples from $\N\left(\mu^{(r)}_j, \alpha / T^{(r)}_j\right)$ for $j \in I^{(r)}, j \neq i$. We then prune all arms $i$ whose corresponding $q^{(r)}_i$ is smaller than $q^{(r)}$, where $q^{(r)}$ is the maximum value among all $q^{(r)}_i\ (i \in [N])$.  The main difference between Algorithm~\ref{alg:independent} and Algorithm~\ref{alg:dependent} lies in the number of pulls on each surviving arm in each batch: In Algorithm~\ref{alg:independent}, the values $T^{(r)}$ are chosen such that the lengths of batches grow much faster.  In each batch of Algorithm~\ref{alg:independent}, the number of pulls on each arm $i$ is again proportional to $q^{(r)}_i$.

\begin{algorithm}[t]
\caption{Batched-TS-I$(N, T, \alpha, \beta, M)$}
\label{alg:independent}
\KwIn{A set of $N$ arms $\{1, \ldots, N\}$, time horizon $T$, \#batches $M$, parameters $\alpha, \beta$}

	Initialize $I^{(1)} \gets I$\;
	Pull each arm once; For each $i \in [N]$, $T^{(1)}_i \gets 1$, and $\mu^{(1)}_i \gets \text{reward of pulling arm $i$}$ \label{ln:init}\;
	Let $a \gets (T - N)^{\frac{1}{2 - 2^{1-M}}}$ \;
	Set $u_1 \gets a$, and $u_r \gets a \sqrt{u_{r-1}}$ for $r = 2, \dotsc, M$ \;
	Set $T^{(1)} \gets N$, and $T^{(r)} \gets \lfloor u_{r - 1}\rfloor + N$ for $r = 2, \dotsc, M+1$\;
	\For{$r = 1, \dotsc, M$}{
		$\forall i \in I^{(r)}$, let $X^{(r)}_i$ denote the random variable distributed according to $\N\left(\mu^{(r)}_i, \alpha / T^{(r)}_i\right)$ \;
		\lForEach{$i \in I^{(r)}$}{$q^{(r)}_i \gets \Pr\left[X^{(r)}_i > \max\limits_{j \in I^{(r)}, j \neq i} X^{(r)}_j\right]$}
		$q^{(r)} \gets \max\limits_{i \in I^{(r)}} q^{(r)}_i$ \;
		$I^{(r + 1)} \gets \left\{q^{(r)}_i \geq q^{(r)} /\beta \mid i \in I^{(r)}\right\}$ \label{ln:prune-2} \;
		\lForEach{$i \in I^{(r + 1)}$}{$\hat{q}^{(r)}_i = q^{(r)}_i / \sum\limits_{j \in I^{(r + 1)}} q^{(r)}_j$}
		\ForEach{$i \in I^{(r  +1)}$}{
			pull arm $i$ for $\hat{q}^{(r)}_i \cdot (T^{(r + 1)} - T^{(r)})$ times; set $T_i^{(r + 1)} \gets T^{(r)}_i + \hat{q}^{(r)}_i \cdot (T^{(r + 1)} - T^{(r)})$ \;
			compute $\mu^{(r + 1)}_i$ using $\mu^{(r)}_i$, $T^{(r)}_i$ and the pull outcomes on arm $i$ in the current batch \;
		}
	}
\end{algorithm}

\vspace{2mm}
\noindent{\bf Analysis.\ \ }  Our theoretical analysis for Algorithm~\ref{alg:independent} is again for the case when $N=2$.   
\smallskip

\begin{theorem}
\label{thm:independent}
	Setting $\alpha = \ln(2T)$ and $\beta = 100$, Batched-TS-I$(2, T, \alpha, \beta, M)$ (Algorithm~\ref{alg:independent}) has the expected regret at most $O\left(T^{1/(2 - 2^{1-M})} \sqrt{\ln T}\right)$ and uses at most $M+1$ batches.
\end{theorem}
\begin{proof}
	The proof for Theorem~\ref{thm:independent} is similar to that for Theorem~\ref{thm:dependent}, and we will reuse some notations and lemmas.
	
	We again condition on the event $\E$ defined in (\ref{eq:E}), and focus on the quantity $\bE[\text{Reg}_{T}(\pi) \mid \E]$.  W.l.o.g., we still assume arm $1$ is the best arm.  By the pruning step at Line~\ref{ln:prune-2} of Algorithm~\ref{alg:independent}, we again have that in any batch $r$ when there are still two arms left, we have $T_1^{(r)}/T_2^{(r)} \in [1/\beta, \beta]$.
	
	Recall that $\Delta \triangleq \Delta_2$ is defined to be the gap between the means of the two arms.  In the case when $\Delta < \sqrt{\frac{100\beta \ln(2T)}{T}}$,  the regret can be bounded by 
	\begin{eqnarray*}
		\text{Reg}_T(\pi) &\le& T \Delta \le \sqrt{100 \beta T \ln(2T)} \\
		&=& O\left(T^{1/(2 - 2^{1-M})} \sqrt{\ln T}\right)
	\end{eqnarray*} 
	for any integer $M \ge 1$.
	
	We next consider the case when $\Delta \ge \sqrt{\frac{100\beta \ln(2T)}{T}}$.
	Consider the number of pulls on arm $2$.   If $T_2^{(r)} \le \frac{100\beta \ln(2T)}{\Delta^2}$ for any batch $r$, then the regret is bounded by 
	\begin{eqnarray*}
		\Delta T^{(r)}_2 &\le& \frac{100 \beta \ln(2T)}{\Delta} \le \sqrt{100 \beta T \ln(2T)} \\
		&=& O\left(T^{1/(2 - 2^{1-M})} \sqrt{\ln T}\right)
	\end{eqnarray*} 
	for any integer $M \ge 1$.
	
	Otherwise, let $m$ be the smallest index such that
	$T^{(m)}_2 > \frac{100 \beta \ln(2T)}{\Delta^2}$. We thus have $T^{(m-1)}_2 \le \frac{100 \beta \ln(2T)}{\Delta^2}$, and consequently, 
	\begin{eqnarray*}
		T^{(m-1)} &=& T^{(m-1)}_1 + T^{(m-1)}_2 \\
		&\le& (1+\beta) \frac{100 \beta \ln(2T)}{\Delta^2} \\ 
		&\le&  \frac{200 \beta^2 \ln(2T)}{\Delta^2}.
	\end{eqnarray*}
	By our choices of $T^{(r)}$ we have 
	\begin{eqnarray*}
		T^{(m)} &\le& a \sqrt{T^{(m - 1)}} \le a \cdot \frac{20 \beta \sqrt{\ln(2T)}}{\Delta} \\
		&\le& 20\beta T^{{1}/{(2-2^{1-M})}} \cdot \frac{\sqrt{\ln(2T)}}{\Delta}.
	\end{eqnarray*}
	
	It is easy to see that Lemma~\ref{lem:X1} still holds in the setting of Algorithm~\ref{alg:independent}.  We thus have that arm $1$ will never be pruned.  Next, by Lemma~\ref{lem:pull-bound}, after the $m$-th batch, arm $2$ will be pruned. Therefore, the total regret of Algorithm~\ref{alg:independent} can be bounded by
	\begin{equation*}
		\Delta \cdot 20\beta T^{\frac{1}{2-2^{1-M}}} \cdot \sqrt{\ln(2T)} \cdot {1}/{\Delta} = O\left(T^{{1}/{(2-2^{1-M})}} \sqrt{\ln T}\right).
	\end{equation*}
	
	Finally, it is clear that the total number of batches used by Algorithm~\ref{alg:independent} is $(M+1)$ where the extra batch is due to the initialization step (Line~\ref{ln:init}).
\end{proof}

Setting $M = \log\log T - 1$, we have the following corollary.
\begin{corollary}
\label{cor:independent}
	Batched-TS-I$(2, T, \ln(2T), 100, \log\log T - 1)$ has the expected regret at most $O(\sqrt{T\ln T})$ and uses at most $\log\log T$ batches.
\end{corollary}

\section{Experiments}
\label{sec:exp}

We have implemented and tested {\tt Batched-TS-D} (Algorithm~\ref{alg:dependent}, \TSD\ for short) and {\tt Batched-TS-I} (Algorithm~\ref{alg:independent}, \TSI\ for short).  By default, we set parameters $\alpha = 1$, $\beta = 100$, and round parameter $M = 20$. We  slightly modified \TSI\ by removing the initialization batch and perform the same number of pulls on each arm in the subsequent batch. The motivation for such a modification is that the information about empirical means of the arms collected by the initialization batch (one pull on each arm) is very limited, and thus the subsequent first Thompson sampling batch is close to a uniform pull. With this modification we can save one batch of computation.  

We note that in Theorem~\ref{thm:dependent} and \ref{thm:independent} we have set $\alpha = \ln(NT)$ for the two algorithms to facilitate the theoretical analysis (more precisely, to use the union bound in Lemma~\ref{lem:event-E}), but in the experimental study we found that $\alpha = 1$, which is also the parameter used in the standard Thompson sampling, already gives good performance. A more careful tuning of the parameters $\alpha, \beta$ will give better performance for \TSD\ and \TSI\ on individual datasets, but we choose not to optimize it for the sake of demonstrating the robustness of our proposed algorithms. 

As mentioned previously, we also test the performance of the two algorithms if we remove their pruning steps.  We denote the two corresponding algorithms \TSDv\ and \TSIv.

We compare \TSD\ and \TSI\ with two sets of existing algorithms.  The first set is two fully sequential algorithms.
(1) \TS  \cite{AG12}: the standard Thompson sampling algorithm for multi-armed bandits. 
(2) \UCB \cite{ACF02}: the standard upper bound confidence algorithm for multi-armed bandits.

The second set is four batched algorithms. For all these algorithms, we set the parameters such that the maximum number of batches is $20$.  
(1) \BSE \cite{GHRZ19}: an elimination based algorithm with instance-dependent bounds using $\Theta(\log T)$ batches.  
(2) \BAE \cite{GHRZ19}: an elimination based algorithm with instance-independent bounds using $\Theta(\log \log T)$ batches.
(3) \UCBI \cite{AO10}: the {\tt improved-UCB} algorithm in \cite{AO10} which can naturally be implemented in the batched setting with $\Theta(\log T)$ batches.
(4) \KMS\ \cite{KMS21}: a batched Thompson sampling algorithm with Gaussian prior using $\Theta(N \log T)$ batches.  


For all the tested algorithms, once there is only one arm left after some batch (that is, the algorithm ``converges'' to the best arm), we make the rest pulls a {\em single} batch.

\vspace{2mm}
\noindent{\bf Datasets and Experimental Environments.\ \ }  We use a combination of synthetic and real datasets. In all datasets we have a set of arms under Bernoulli distributions with means $\mu_i\ (i \in [N])$.  In the first three datasets we have two arms ($N = 2$) with the following means.
\begin{itemize}
	\item[] DS1: $\mu_1 = 0.9$ and $\mu_2 = 0.6$;
	\item[] DS2: $\mu_1 = 0.9$ and $\mu_2 = 0.8$;
	\item[] DS3: $\mu_1 = 0.55$ and $\mu_2 = 0.45$.
\end{itemize}
In the next three datasets we have ten arms ($N = 10$) with the following means.
\begin{itemize}
	\item[] DS4 (one-shot): $\mu_1 = 0.9$, and for $\forall{i} \neq 1 : \mu_i = 0.8$;
	\item [] DS5 (uniform): $\mu_i = 0.9 - (i - 1) / 20$;
	\item [] DS6 (clustered): $\mu_1 = 0.9$, $\mu_2 = \mu_3 = \mu_4 = 0.8$, $\mu_5 = \mu_6 = \mu_7 = 0.7$, $\mu_8 = \mu_9 = \mu_{10} = 0.6$.
\end{itemize}
We also test on a real-world dataset MovieLens~\cite{HK16}.
\begin{itemize}
	\item [] MOVIE: 
We select the movies scored by at least $20,000$ users; there are 588 such movies. For the $i$-th movie, we set $\mu_i \in [0,1]$ as the average rating divided by $5$.
\end{itemize} 

\begin{figure*}[h!t]
	\begin{center}
		\includegraphics[width=.32\textwidth]{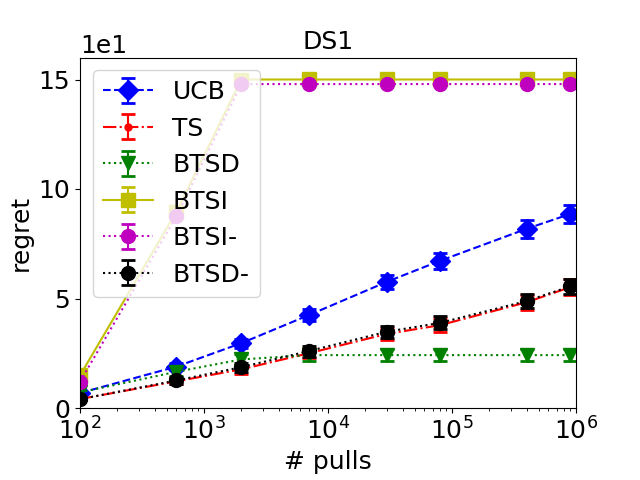}
		\includegraphics[width=.32\textwidth]{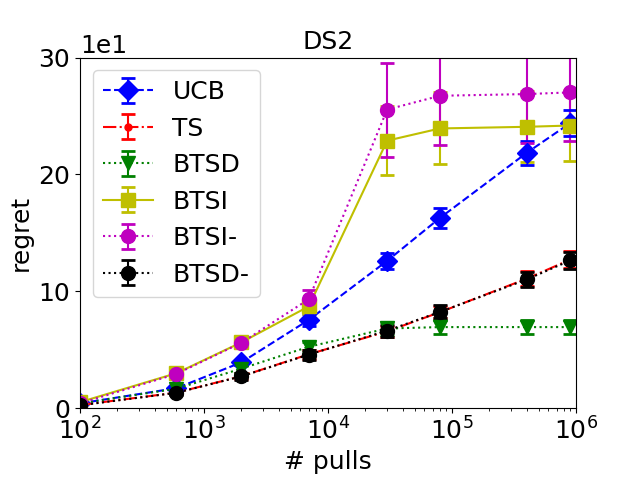}
		\includegraphics[width=.32\textwidth]{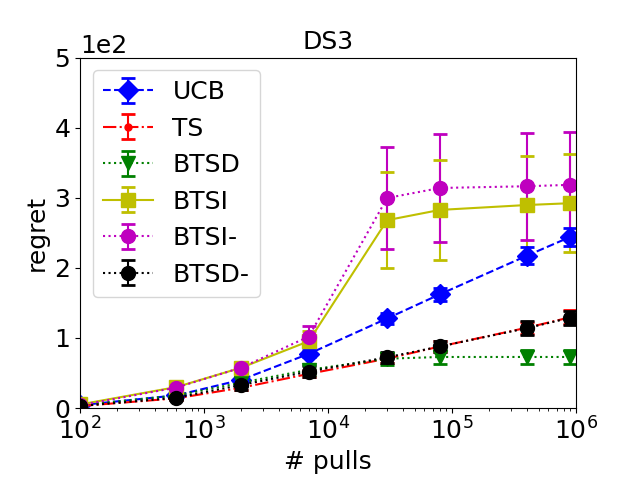}
		\includegraphics[width=.32\textwidth]{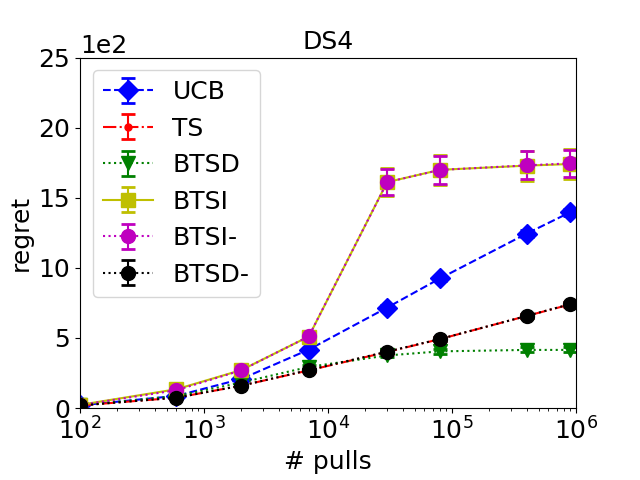}
		\includegraphics[width=.32\textwidth]{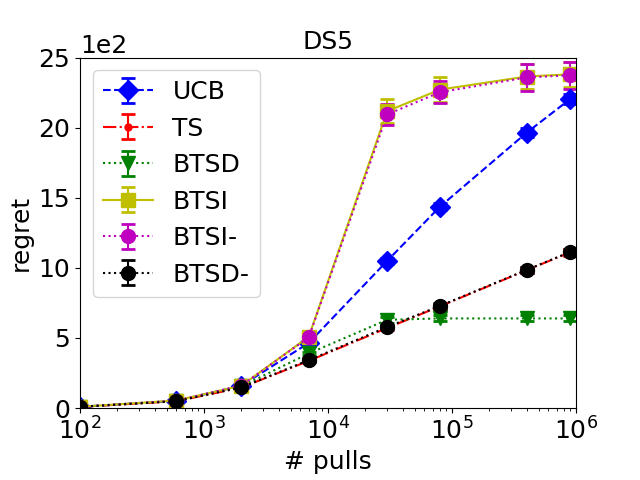}
		\includegraphics[width=.32\textwidth]{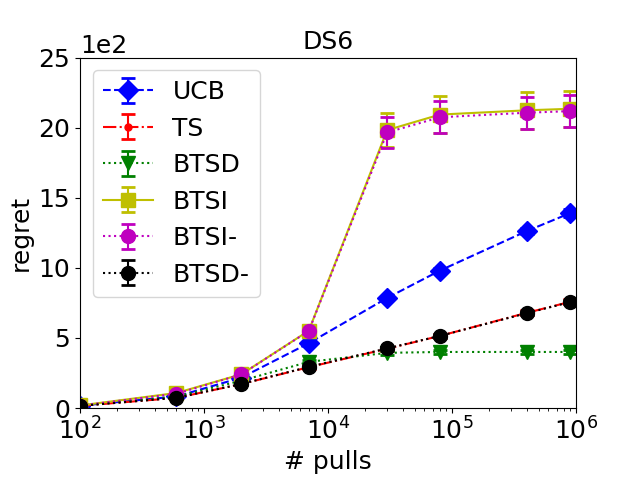}
	\end{center}
	\caption{Regret on DS1 -- DS6. Curves for \TSDv\ and \TS\ almost overlap; curves for \TSI\ and \TSIv\ almost overlap.}\label{fig:exp1}
\end{figure*}

\begin{figure*}[h!t]
	\begin{center}
		\includegraphics[width=.32\textwidth]{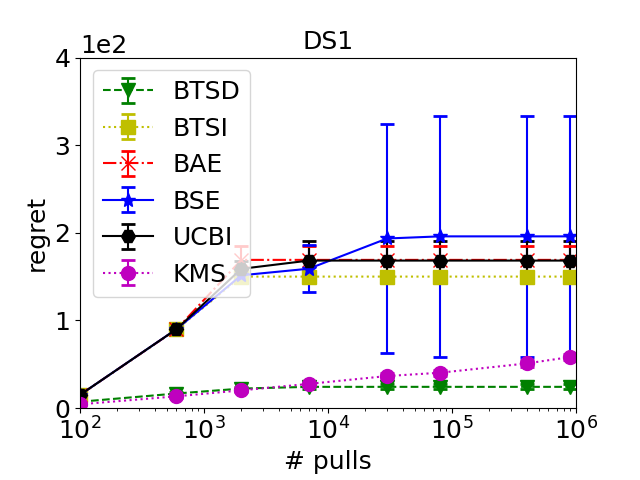}
		\includegraphics[width=.32\textwidth]{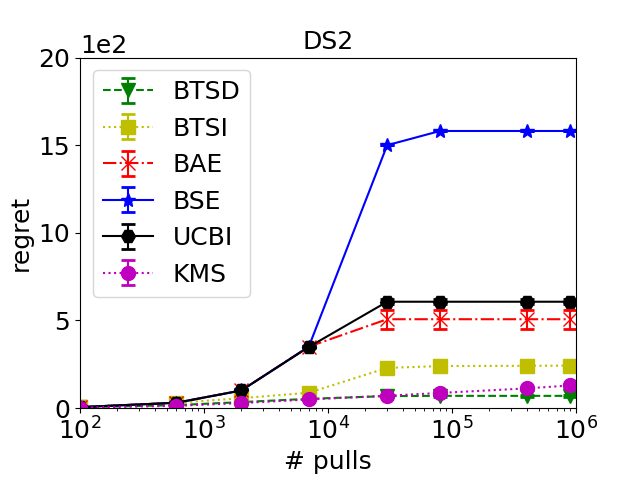}
		\includegraphics[width=.32\textwidth]{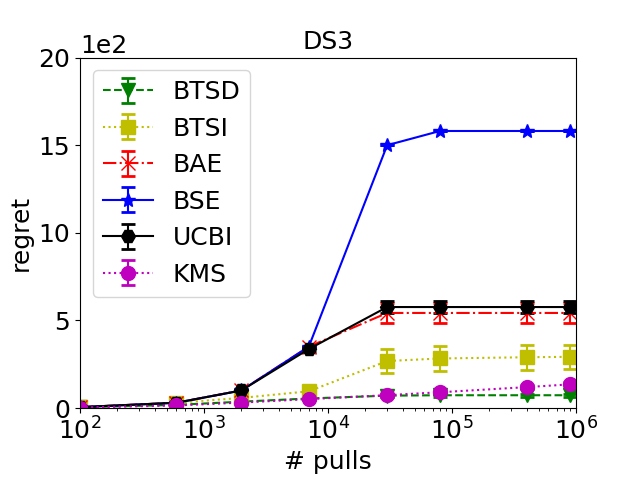}
		\includegraphics[width=.32\textwidth]{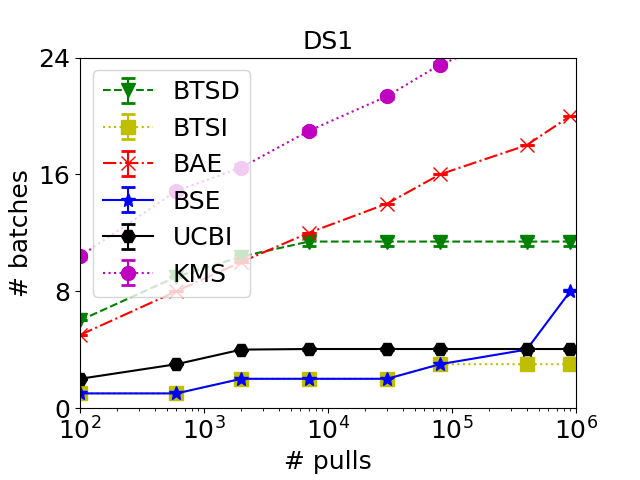}
		\includegraphics[width=.32\textwidth]{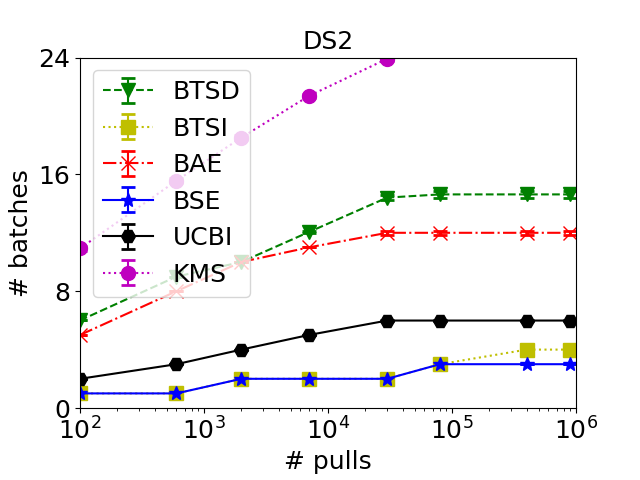}
		\includegraphics[width=.32\textwidth]{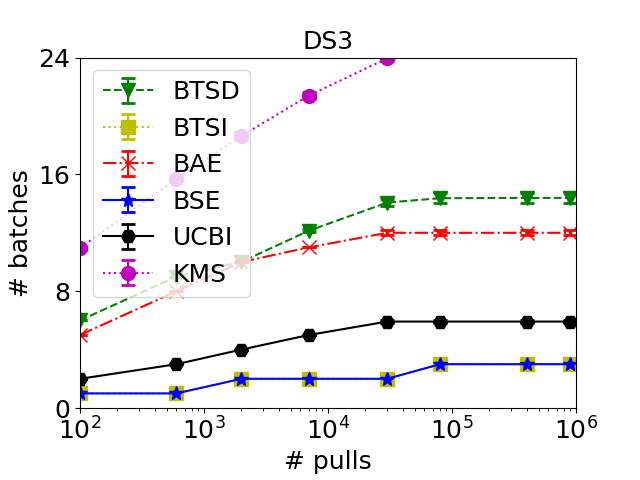}
	\end{center}
	\caption{Regret and \#batch on DS1 -- DS3. Curves for batch costs of \TSI\ and \BSE\ almost overlap}\label{fig:exp2}
\end{figure*}

We measure the regret and the number of batches used by the aforementioned algorithms over a sequence of arm pulls on these datasets. The results are average over $100$ runs.  Error bars measure the standard deviation.

All algorithms were implemented using the Kotlin programming language. All experiments were conducted in PowerEdge R740 server equipped with  $2 \times$Intel Xeon Gold 6248R 3.0GHz (24-core/48-thread per CPU) and 256GB RAM.

\vspace{1mm}
\noindent{\bf Experiments and Results.\ \ }  Our experiments consist of three parts.  We first compare our algorithms \TSI/\TSIv\ and \TSD/\TSDv\ with \TS\ and \UCB\ in the sequential setting on datasets DS1-DS6.  The results are described in Figure~\ref{fig:exp1}. We have the following assumptions.  The performance of \TSDv\ is close to \TSD, and is almost identical with the standard {\em fully adaptive} Thompson sampling \TS. The performance of \TSD\ is even better than \TS, and is much better than \UCB; the former may due to the active pruning step in \TSD.  The performance of \TSIv\ is almost identical with \TSI; both are worse than \TSD.  This is mainly because \TSI\ uses fewer batches and larger batch sizes at the beginning; we will present their batch costs in the next set of experiments.

\begin{figure*}[h!t]
	\begin{center}
		\includegraphics[width=.32\textwidth]{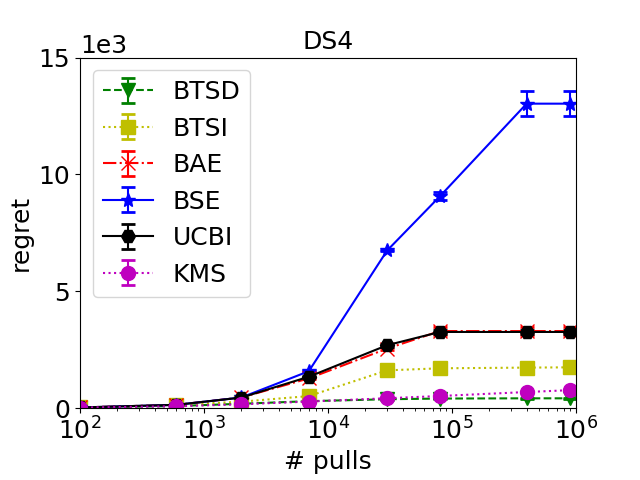}
		\includegraphics[width=.32\textwidth]{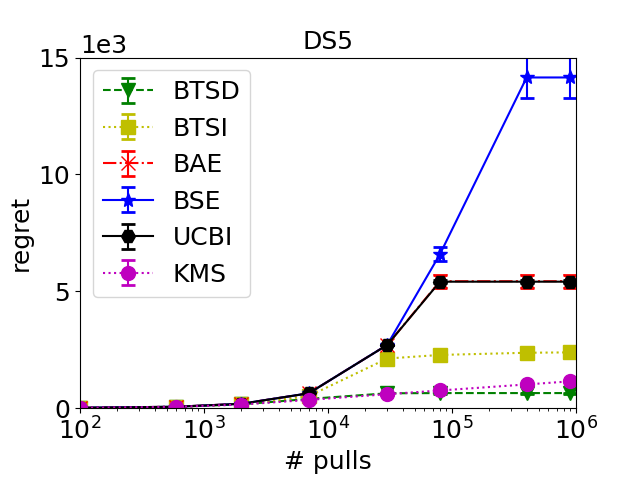}
		\includegraphics[width=.32\textwidth]{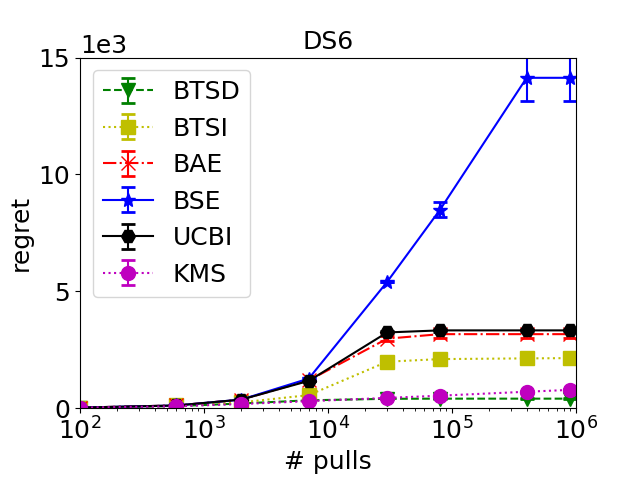}		
		\includegraphics[width=.32\textwidth]{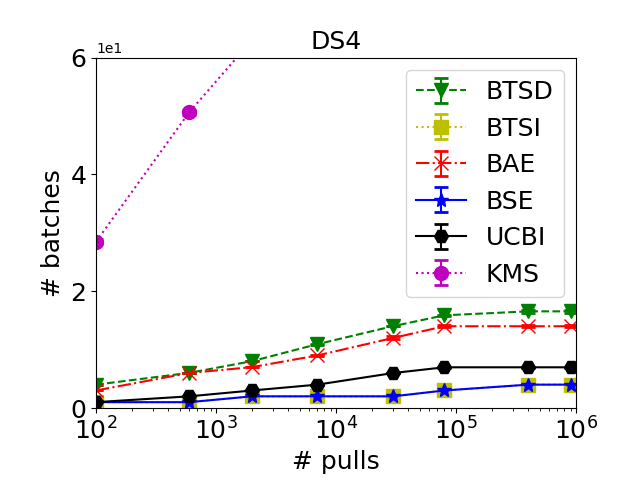}
		\includegraphics[width=.32\textwidth]{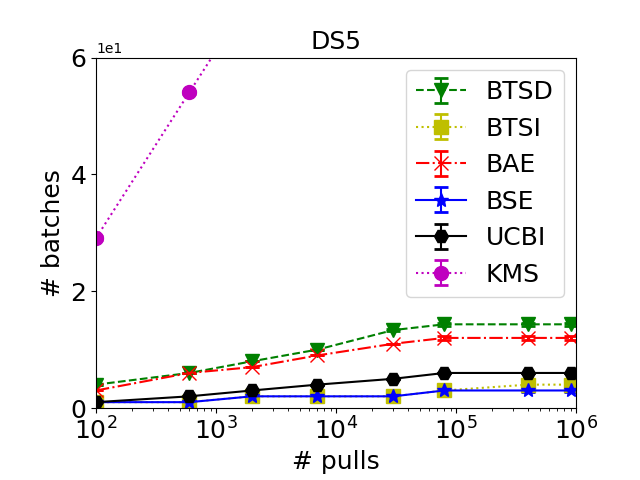}
		\includegraphics[width=.32\textwidth]{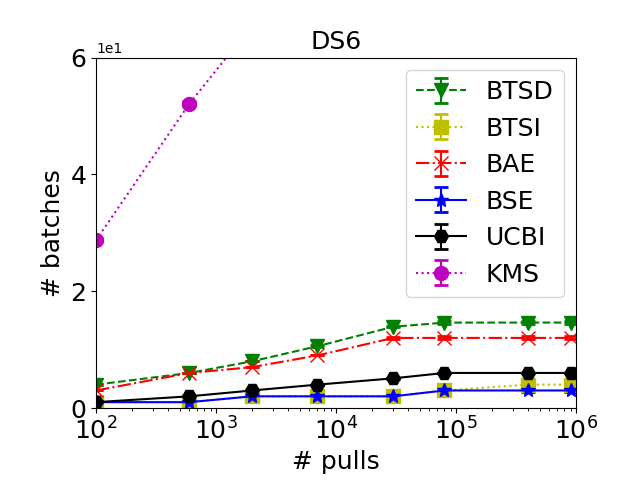}
	\end{center}
	\caption{Regret and \#batch on DS4 -- DS6. Curves for batch costs of \TSI\ and \BSE\ almost overlap}\label{fig:exp3}
\end{figure*}

We next compare our algorithms \TSI\ and \TSD\ with \BSE, \BAE, \UCBI, and \KMS\ in the batched setting on DS1-DS6.  The results are described in Figure~\ref{fig:exp2} and \ref{fig:exp3}.   We observe that in terms of regret, both \TSI\ and \TSD\ significantly outperform \BSE, \BAE, and \UCBI.  Similar to that in Figure~\ref{fig:exp1}, \TSD\ outperforms \TSI\ in regret.  The performance of \KMS\ is between \TSD\ and \TSI.  Regarding the batch cost, we observe that \TSI\ performs almost the same as \BSE; the two have the smallest batch costs.  \TSD\ performs similarly to \BAE. These phenomena are not surprising since in both \TSD\ and \BAE\ the batch size increases single exponentially, while in both \TSI\ and \BSE\ the batch sizes are designed essentially the same and increase much faster than single exponential.  The performance of \UCBI\ is always in the middle, and have the same trend as \TSD\ and \BSE. This is because the batch size of \UCBI\ also grows single exponentially but with different constant parameters.  The batch complexity of \KMS\ is much larger than others.

\begin{figure*}[h!t]
	\begin{center}
		\includegraphics[width=.32\textwidth]{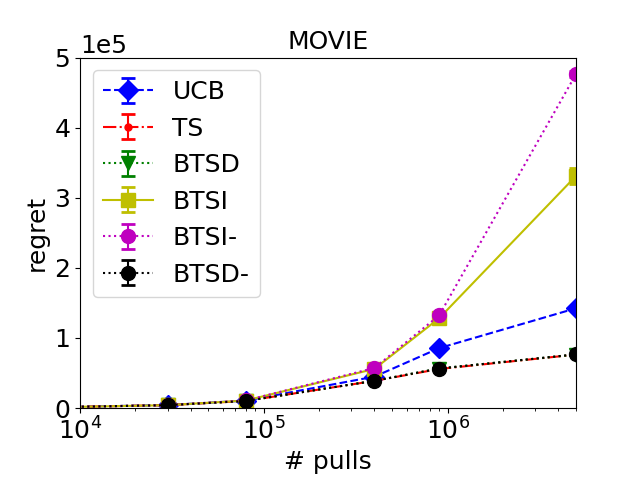}
		\includegraphics[width=.32\textwidth]{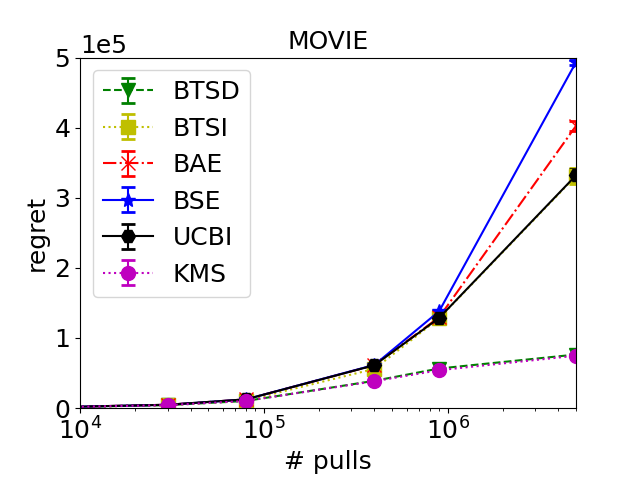}
		\includegraphics[width=.32\textwidth]{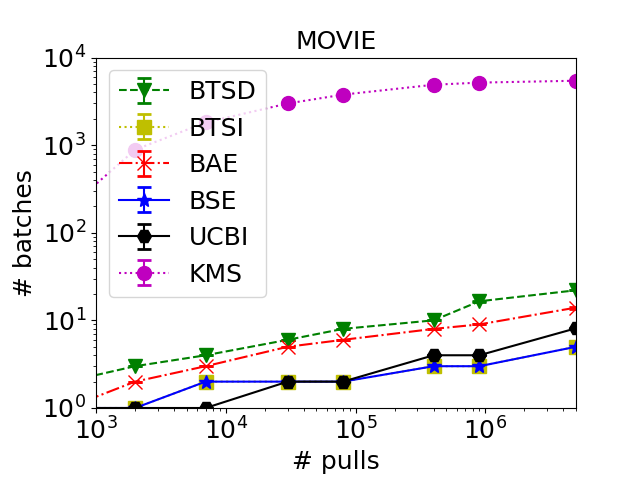}
	\end{center}
	\caption{Regret and \#batch on MOVIE. In the left figure, curves for \TSD\ , \TSDv\ and \TS\ almost overlap; in the middle figure, curves for \TSI\ and \UCBI\ almost overlap, and curves for \TSD\ and \KMS\ almost overlap; in the right figure, curves for \TSI\ and \BSE\ overlap}\label{fig:exp4}
\end{figure*}

Finally, we compare the algorithms on the real world dataset MOVIE.  The results are described in Figure~\ref{fig:exp4}.  In the sequential setting (left subfigure), \TSD\ performs similarly to \TS, and outperforms \UCB\ and \TSI.  Again, \TSDv\ performs similarly to \TSD, and \TSIv\ perform similarly (except for the last point) to \TSI. In the batched setting (middle and right subfigures), \TSD\ and \KMS\ significantly outperforms competitors in regret. But like that on synthetic datasets, the batch cost of \KMS\ is much larger than \TSD\ and other competitors \TSI\ performs similarly to \UCBI\ in regret, and does better than \BSE\ and \BAE.   The batch cost of \TSI\ is still similar to \BSE, and is better than other competitors.

\vspace{1mm}
\noindent{\bf Conclusion.\ \ }  We have observed the followings from our experiments: (1) Our proposed algorithm \TSD\ slightly outperforms \KMS, and significantly outperforms all other state-of-the-art batched algorithms. But \KMS\ has a significant larger batch cost than all competitors. \TSD\ also performs similarly to (or better than) the standard, fully adaptive, Thompson sampling algorithm \TS\ in regret.  (2) \TSI\ and the existing algorithm \BSE\ have the smallest batch cost, but \TSI\ significantly outperforms \BSE\ in regret.

\section{Concluding Remarks}
\label{sec:conclude}

In this paper, we study batched Thompson sampling for regret minimization in stochastic multi-armed bandits. We propose two algorithms and demonstrate experimentally their superior performance compared with state-of-the-art in the batched setting.  We provide rigorous theoretical analysis for the special case when there are two arms, and show that our algorithms achieve almost optimal (up to a logarithmic factor) regret-batches tradeoffs.  Due to some technical challenges, we are not able to prove tight bounds for general $N$, though we believe this is achievable (by properly choosing the parameters $\alpha, \beta$); we leave it as future work.

\balance

\newpage

\bibliographystyle{plain}
\bibliography{paper}

\newpage

\appendix

\end{document}